\documentclass{llncs}
\usepackage{llncsdoc}
\usepackage{hyperref}
\usepackage{times}
\usepackage{color}
\usepackage{graphicx}
\usepackage{amscd,amsmath,amssymb}
\usepackage{verbatim}
\usepackage[ruled]{algorithm2e}
\usepackage{multicol,multienum}

\newcommand{\sfp}{\textsf{b}}
\newcommand{\sfm}{\textsf{m}}
\newcommand{\sfo}{\textsf{o}}
\newcommand{\sfd}{\textsf{d}}
\newcommand{\sfs}{\textsf{s}}
\newcommand{\sff}{\textsf{f}}
\newcommand{\sfeq}{\textsf{eq}}
\newcommand{\sffi}{\textsf{fi}}
\newcommand{\sfsi}{\textsf{si}}
\newcommand{\sfoi}{\textsf{oi}}
\newcommand{\sfdi}{\textsf{di}}
\newcommand{\sfmi}{\textsf{mi}}
\newcommand{\sfpi}{\textsf{bi}}

\newcommand{\bdc}{{\bf DC}}

\newcommand{\bec}{{\bf EC}}
\newcommand{\beq}{{\bf EQ}}

\newcommand{\bpo}{{\bf PO}}

\newcommand{\btpp}{{\bf TPP}}

\newcommand{\bntpp}{{\bf NTPP}}

\newcommand{\qcm}{\mathcal{M}}
\newcommand{\opra}{\mathcal{OPRA}}
\newcommand{\mst}[3]{_{#1}\angle^{#3}_{#2}}

\newcommand{\step}{\textsc{Loop}}
\newcommand{\triad}{\textsc{Triad}}
\newcommand{\last}{\textsc{LastFound}}

\title{On A Semi-Automatic Method for Generating Composition Tables
}

\author{Weiming Liu \and Sanjiang Li}

\institute{Centre for Quantum Computation and Intelligent Systems,
       Faculty of Engineering and Information Technology, University of Technology
       Sydney, Australia}

\begin{document}

\maketitle

\begin{abstract}
Originating from Allen's Interval Algebra, composition-based reasoning has been widely acknowledged as the most popular reasoning technique in qualitative spatial and temporal reasoning. Given a qualitative calculus (i.e. a relation model), the first thing we should do is to establish its composition table (CT). In the past three decades, such work is usually done manually. This is undesirable and error-prone, given that the calculus may contain tens or hundreds of basic relations. Computing the correct CT has been identified by Tony Cohn as a challenge for computer scientists in 1995. This paper addresses this problem and introduces a semi-automatic method to compute the CT by randomly generating triples of elements. For several important qualitative calculi, our method can establish the correct CT in a reasonable short time. This is illustrated by applications to the Interval Algebra, the Region Connection Calculus RCC-8, the INDU calculus, and the Oriented Point Relation Algebras. Our method can also be used to generate CTs for customised qualitative calculi defined on restricted domains.


\end{abstract}

\section{Introduction}
Since Allen's seminal work of Interval Algebra (IA)  \cite{Allen81,Allen83}, qualitative calculi have been widely used to represent and reason about temporal and spatial knowledge. In the past decades, dozens of qualitative calculi have been proposed in the artificial intelligence area ``Qualitative Spatial \& Temporal Reasoning" and Geographic Information Science. Except IA, other well known binary qualitative calculi include the Point Algebra \cite{VilainK86}, the Region Connection Calculi RCC-5 and RCC-8 \cite{RandellCC92}, the INDU calculus \cite{PujariKS99}, the Oriented Point Relation Algebras $\opra$ \cite{Moratz06}, and the Cardinal Direction Calculus (CDC) \cite{Goyal00}, etc.

Relations in each particular qualitative calculus are used to represent temporal or spatial information at a certain granularity.  For example, The Netherlands is \emph{west} of Germany, The Alps \emph{partially overlaps} Italy, I have today an appointment with my doctor \emph{followed by} a check-up.

Given a set of qualitative knowledge, new knowledge can be derived by using constraint propagation. Consider an example in RCC-5. Given that The Alps \emph{partially overlaps} Italy and Switzerland, and Italy is a \emph{proper part of} the European Union (EU), and Switzerland is \emph{discrete from} the EU, we may infer that The Alps \emph{partially overlaps} the EU. The above inference can be obtained by using composition-based reasoning.  The composition-based reasoning technique has been extensively used in qualitative spatial and temporal reasoning, and, when combined with backtracking methods, has been shown to be complete in determining the consistency problem for several important qualitative calculi, including IA, Point Algebra, Rectangle Algebra, RCC-5, and RCC-8. Moreover, qualitative constraint solvers have been developed to facilitate composition-based reasoning \cite{WallgrunFWDF06,WestphalWG09}.

We here give a short introduction of the composition-based reasoning technique. Suppose $\qcm$ is a qualitative calculus, and $\Gamma=\{v_i\gamma_{ij} v_j\}_{i,j=1}^n$ is a constraint network over $\qcm$. The composition-based reasoning technique uses a variant of the well-known \emph{Path Consistency Algorithm},\footnote{The notion of Path Consistency is usually defined for constraints on finite domains, and not always appropriate for general qualitative constraints, which are defined on infinite domains.}  which applies the following updating rule until the constraint network becomes stable or an empty relation appears:
\begin{equation}\label{eq:pca}
\gamma_{ij}\leftarrow \gamma_{ij}\cap \gamma_{ik}\circ_w\gamma_{kj},
\end{equation}
where $\alpha\circ_w\beta$ is the \emph{weak composition}  (cf. \cite{LiY03,RenzL05}) of two relations $\alpha,\beta$ in $\qcm$, namely the smallest relation in $\qcm$ which contains the usual composition of $\alpha$ and $\beta$.  Although for OPRA and some other calculi the composition-based reasoning is incomplete to decide the consistency problem, it remains a very efficient method to approximately solve the consistency problem.

The weak composition in a qualitative calculus $\qcm$ is determined by its \emph{weak composition table} (CT for short). Usually, the CT of $\qcm$ is obtained by manually checking the consistency of $\{x\alpha y, y\beta z, x\gamma z\}$ for each triple of basic relations $\langle\alpha,\gamma,\beta\rangle$. When  $\qcm$ contains dozens or even hundreds of basic relations, this consistency-based method is undesirable and error-prone.   \cite{Cohn95} first noticed this problem and identified it as a challenge for computer scientists.

This problem remains a challenge today. We here consider several examples. The Interval Algebra and the RCC-8 algebra contain, respectively, 13 and 8 basic relations. Their CTs were established manually. But if a calculus contains  a hundred basic relations, we need to determine the consistency of one million such basic networks. This is manually impossible. The $\opra$ calculi and the CDC are large qualitative spatial calculi that have drawn increasing interests.  $\opra_m$ contains $4m\times (4m+1)$ (i.e. 72, 156, 272 for $m=2,3,4$, respectively) basic relations \cite{Moratz06}, while the CDC contains 218 basic relations \cite{Goyal00}.  Sometimes we need ingenious and special methods to establish CT for such a calculus. For the $\opra$ calculi,  the algorithm presented in the original paper \cite{Moratz06} contains gaps and errors.  Later, \cite{Frommberger+07} presented the second algorithm, which is quite lengthy and cumbersome. Another simple algorithm has also been proposed recently \cite{MM10}. Given the huge number of basic relations of $\opra_m$, the validity of these algorithms need further verification. As for the CDC, \cite{Goyal00} first studied the weak composition. Later, \cite{SkiadopoulosK04} noticed errors in Goyal's method and gave a new algorithm to compute the weak composition. Unfortunately, in several cases, their algorithm does not generate the correct weak composition (see \cite{LiuZLY10}).

In this paper, we respond to this challenge and propose a semi-automatic approach to generate CT for general qualitative calculi. In the remainder of this paper, we first recall basic notions and results about qualitative calculi and weak composition tables in Section 2, and then apply our method to IA, INDU, RCC-8, and $\opra_1$ and $\opra_2$ in Section 3. An analysis of the strength and weakness of our approach is given in Section~4. Section~5 then concludes the paper.

\section{Preliminaries}
In this section we recall the notions of qualitative calculi and their weak composition tables. Interested readers may consult e.g. \cite{LigozatR04,RenzL05} for more information.


\begin{definition}\label{dfn:qc}
Suppose $U$ is a universe of spatial or temporal entities, and $\mathcal{B}$ is a set of jointly exhaustive and pairwise disjoint (JEPD) binary relations on $U$. We call the Boolean algebra generated by $\mathcal{B}$ a \emph{qualitative calculus}, and call relations in $\mathcal{B}$ the \emph{basic relations} of this qualitative calculus.
\end{definition}

We consider a simple example.
\begin{example}[Point Algebra]
Suppose $U=\mathbb{R}$. For two points $a,b$ in $U$, we have either $a<b$, or $a=b$, or $a>b$. Let $\mathcal{B}=\{<,=,>\}$. Then $\mathcal{B}$ is a JEPD set of relations on $U$. We call the Boolean Algebra generated by $\mathcal{B}$ the Point Algebra.

\end{example}


We next recall the central notion of weak composition.


\begin{definition}\label{dfn:wcomposition}
Suppose $\qcm$ is a qualitative calculus on $U$, and $\mathcal{B}$ is the set of its basic relations. The \emph{weak composition} of two basic relations $\alpha$ and $\beta$ in $\qcm$, denoted as $\alpha\circ_w\beta$, is defined as the smallest relation in $\qcm$ which contains $\alpha\circ\beta$, the usual composition of $\alpha$ and $\beta$.
\end{definition}

Usually, a qualitative calculus has a finite set of relations. The weak composition operation of $\qcm$ can be summarised in an $n\times n$ table, where $n$ is the cardinality of $\mathcal{B}$, and the cell specified by $\alpha$ and $\beta$ contains all basic relations $\gamma$ in $\mathcal{B}$ such that  $\gamma\cap\alpha\circ\beta\neq\varnothing$.
The CT of the Point Algebra is given in Table~\ref{tab:PA}.
\begin{table}[htb]
\centering
\caption{The CT of the Point Algebra, where $\ast$ is the universal relation}\label{tab:PA}
\scalebox{1}{
\begin{tabular}{c|ccc}
  $\circ$ & $<$ & $=$ & $>$ \\  \hline
  $<$ & $<$ & $<$ & $\ast$ \\
  $=$ & $<$ & $=$ & $>$ \\
  $>$ & $\ast$ & $>$ & $>$
\end{tabular}
}

\end{table}

\begin{figure}[htb]
\centering
\includegraphics[width=.3\textwidth]{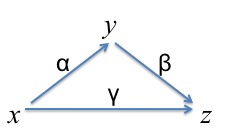}
\caption{A c-triad $\langle\alpha,\gamma,\beta\rangle$}
\label{fig:c-triad}
\end{figure}
\begin{definition}\label{dfn:c-triad}
Suppose $\qcm$ is a qualitative calculus on $U$ with basic relation set $\mathcal{B}$. For basic relations $\alpha,\beta,\gamma$, we call $\langle\alpha,\gamma,\beta\rangle$ a \emph{composition triad}, or \emph{c-triad}, if $\gamma\subseteq\alpha\circ_w\beta$.
\end{definition}
We can determine if a 3-tuple is a c-triad as follows (cf. Fig.~\ref{fig:c-triad}).
\begin{proposition}\label{prop:compute-ctriad}
A 3-tuple $\langle\alpha,\gamma,\beta\rangle$ of basic relations in $\qcm$ is a c-triad iff $\gamma\cap \alpha\circ\beta\not=\varnothing$, which is equivalent to saying that the basic constraint network
\begin{equation}\label{eq:c-triad}
\{x \alpha y, y\beta z, x\gamma z\}
\end{equation}
is consistent, i.e. it has a solution in $U$.
\end{proposition}

To compute the weak composition of $\alpha$ and $\beta$, one straightforward method is to find all basic relations $\gamma$ such that $\langle\alpha,\gamma,\beta\rangle$ is a c-triad.



\section{A General Method for Computing CT}
In this section, we propose a general approach to compute the composition table of a qualitative calculus $\qcm$ with domain $U$ and basic relation set $\mathcal{B}$. The approach is based on the observation that each triple of objects in $U$ derives a valid c-triad.
\begin{proposition}\label{lemma:abc}
Suppose $a,b,c$ are three objects in $U$. Then $\langle\rho(a,b),\rho(a,c),\rho(b,c)\rangle$ is a c-triad, where $\rho(x,y)$ is the basic relation in $\qcm$ that relates $x$ to $y$.
\end{proposition}
It is clear that six (different or not) c-triads can be generated if we consider all permutations of $a,b,c$.

To compute the CT of $\qcm$, the idea is to choose randomly a triple of elements in $U$ and then compute and record the c-triads related to these objects in a dynamic table. Continuing in this way, we will get more and more c-triads until the dynamic table becomes stable after sufficient large loops. The basic algorithm is given in Algorithm~\ref{algorithm1}, where $D$ is a subdomain of $U$,  $\Psi$ decides when the procedure terminates, \triad\ records the number of c-triads obtained when the procedure terminates, and \last\ records the time when the last triad is first recorded. For a calculus with unknown CT, the condition may be assigned with the form $\step\leq 1,000,000$ (i.e., the algorithm loops one million times), or $\step\leq \last+100,000$ (i.e., until no new c-triad is found in the last one hundred thousand loops), or their conjunction. If the CT is known and we want to double-check it, then the boundary condition could be set to $\triad< N$ to save time, where $N$ is the number of c-triads of the calculus.

\begin{algorithm}[htb]\label{algorithm1}
\caption{Computing the Composition Table of $\qcm$}
\KwIn{A subdomain $D$ of $\qcm$, and a boundary condition $\Psi$ related to $\qcm$}
\KwOut{The Composition Table $CT$ of $\qcm$}
Initialise $CT$\;

$\step\leftarrow 0$;

$\triad\leftarrow 0$;

$\last\leftarrow 0$;

\While{$\Psi$}
{

$\step\leftarrow \step+1$;

Generate triple of objects $(a,b,c)\in D^3$ randomly\;

$\alpha\leftarrow$ the basic relation between $a$ and $b$\;

$\beta\leftarrow$ the basic relation between $b$ and $c$\;

$\gamma\leftarrow$ the basic relation between $a$ and $c$\;

$\alpha^\prime\leftarrow$ the basic relation between $b$ and $a$\;

$\beta^\prime\leftarrow$ the basic relation between $c$ and $b$\;

$\gamma^\prime\leftarrow$ the basic relation between $c$ and $a$\;

\For{$\langle r,s,t\rangle\in\{\langle\alpha,\gamma,\beta\rangle$, $\langle \alpha^\prime,\beta,\gamma\rangle$, $\langle \gamma,\alpha,\beta^\prime\rangle$, $\langle \beta,\alpha^\prime,\gamma^\prime\rangle$, $\langle \beta^\prime,\gamma^\prime,\alpha^\prime\rangle$, $\langle \gamma^\prime,\beta^\prime,\alpha\rangle\}$}
{
\If {$\langle r,s,t\rangle$ is not in $CT$}{
Record triad $\langle r,s,t\rangle$ to $CT$\;
$\triad\leftarrow \triad+1$\;
$\last \leftarrow \step$;
}
}
}
\Return $CT$.

\end{algorithm}

We make further explanations here.

Suppose $\qcm$ is a qualitative calculus on $U$. Recall $U$ is often an infinite set. We need first to decide a finite subdomain $D$ of $U$, as computers only deal with numbers with finite precision.  Once $D$ is chosen, we run the loop, say, one million times. Therefore, one million instances of triples of elements in $D$ are generated. We then record all computed c-triads in a dynamic table. It is reasonable to claim that the table is stable if no new entry has been recorded after a long time (e.g. as long as the time has past to get all recorded c-triads).  Because $D$ is finite, Algorithm~\ref{algorithm1} will generate a stable table after  a sufficient large number of iterations.

We observe that a finite subdomain $D$ may restrict the possible c-triads if it is selected inappropriately. We introduce a notion to characterise the appropriateness of a subdomain.

\begin{definition}\label{dfn:complete-subdomain}
Suppose $\qcm$ is a qualitative calculus defined on the universe $U$. A nonempty subset $D$ of $U$ is called a \emph{3-complete} subdomain of $\qcm$ if each consistent basic network as specified in Eq.~\ref{eq:c-triad} has a solution in $D$.
\end{definition}
If  $D$ is a  3-complete subdomain, then, for each c-triad $\langle\alpha,\gamma,\beta\rangle$, there are $a,b,c$ in $D$ such that $(a,b)\in \alpha$,  $(b,c)\in\beta$, and $(a,c)\in\gamma$. Therefore, to determine the CT of $\qcm$, we need only consider instances of triples in $D$.

Note that no matter whether the subdomain $D$ is 3-complete, the algorithm always generates `valid' triads, in the sense that any 3-tuple $\langle\alpha,\gamma,\beta\rangle$ in the CT generated is indeed a c-triad of the calculus. However, the algorithm only converges to the correct CT when the subdomain $D$ is 3-complete.

It is of course important questions to find 3-complete subdomains or to decide if a particular subdomain is 3-complete.
However, it seems that there is no general answer for arbitrary qualitative calculi, since the questions are closely related to the semantics of the calculi. For a particular calculus, e.g. IA, this can be verified by formal analysis. Note that a superset of a 3-complete subdomain is also 3-complete. To make sure a chosen subdomain $D$ is 3-complete, we often apply the algorithm on several of its supersets at the same time. If the same number is generated for all subdomains, we tend to believe that $D$ is 3-complete and the generated table is the CT of $\qcm$. Note a formal proof is necessary to guarantee the 3-completeness of $D$.

Even if a CT of $\qcm$ has been somehow obtained, our method can be used to verify its correctness. Double-checking is necessary since computing the CT is error-prone (see the last paragraph of page~1). If there is a c-triad that does not appear in the previously given table, something must be wrong with the table, because the c-triads computed by Algorithm~\ref{algorithm1} are always valid. It is also possible that the algorithm terminates with a fragment of given composition table. We then can make theoretical analysis to see if the missing c-triads are caused by the incompleteness of the subdomain. If so, we modify the subdomain and run the algorithm again, otherwise, the missing c-triads are likely to be invalid c-triads.

Another thing we should keep in mind is how to generate a triple of elements  $(a,b,c)$ from $D$. Note that if $D$ is small (e.g. in the cases of PA and IA), we can generate all possible triples. If $D$ contains more than 1000 elements, then it will be necessary to generate the triples randomly as there are over a billion different triples. The distribution over $D$ may affect the efficiency of the algorithm. Assuming that we have very limited knowledge of the calculus $\qcm$, it is natural to take $a,b$ and $c$  independently with respect to the uniform distribution. We note that the better we understand the calculus, the more appropriate the distribution we may choose.

To increase the efficiency of the algorithm, we sometimes use the algebraic properties of the calculus. For example,  if the identity relation $id$ is a basic relation, then by $\alpha\circ_w id=\alpha=id\circ_w \alpha$ and $id \subseteq\alpha\circ_w\alpha^\sim$, we need not compute the c-triads involving $id$, where $\alpha^\sim$ is the converse of $\alpha$. This is to say, the algorithm only needs to generate pairwise different elements. As another example, suppose that the calculus is closed under converse, i.e. the converse of a basic relation is still a basic relation. Then in Algorithm~\ref{algorithm1} we need only compute $\alpha,\beta,\gamma$. The other relations and c-triads can be obtained by replacing $\alpha^\prime,\beta^\prime,\gamma^\prime$ in the algorithm by, respectively, $\alpha^\sim,\beta^\sim,\gamma^\sim$. Similar results have been reported in \cite{Bennett94}.

In the following we examine three important examples. All experiments were conducted on a 3.16 GHZ Intel Core 2 Duo CPU with 3.25 GB RAM running Windows XP.
Note the results rely on the random number generator. As our aim is to show the feasibility of the algorithm rather than investigating the efficiency issues, we only provide one group of the results and do not make any statistical analysis.

\subsection{The Interval Algebra and the INDU Calculus}
We start with the best known qualitative calculus.
\begin{example}[Interval Algebra]
Let $U$ be the set of closed intervals on the real line. Thirteen binary relations between two intervals $x=[x^-,x^+]$ and $y=[y^-,y^+]$ are defined in Table~\ref{tab:int}. The Interval Algebra \cite{Allen83} is the Boolean algebra generated by these thirteen JEPD relations.
\end{example}

\begin{table}[htb]\centering
\caption{Basic IA relations and their converses, where
$x=[x^-,x^+],y=[y^-,y^+]$ are two intervals.}\label{tab:int}
\scalebox{1}{
\begin{tabular}{|c|c|c|c|}
  \hline
  Relation & Symbol & Converse & Meaning  \\ \hline
 before & \sfp & \sfpi & $x^-<x^+<y^-<y^+$  \\
 meets & \sfm & \sfmi &  $x^-<x^+=y^-<y^+$ \\
 overlaps & \sfo & \sfoi & $x^-<y^-<x^+<y^+$ \\
 starts & \sfs & \sfsi & $x^-=y^-<x^+<y^+$ \\
 during & \sfd & \sfdi & $y^-<x^-<x^+<y^+$ \\
 finishes & \sff & \sffi & $y^-<x^-<x^+=y^+$  \\
 equals & \sfeq & \sfeq & $x^-=y^-<x^+=y^+$ \\
  \hline
\end{tabular}
}
\end{table}

\begin{table}[htb]\centering
\caption{Implementation for IA, where $\triad$ is the number of c-triads recorded by running the algorithm on $D_M$ for $M=4$ to $M=20$, $\last$ is the loop when the last triad is first recorded}\label{tab:impl_int}
\begin{tabular}{c|ccccccccccc}
$M$  & 4     & 5     & 6      &  7    & 8      & 9         & 10   &11    & 12\\ \hline
$\triad$ & 139 & 319 & 409  & 409 & 409  & 409     & 409 & 409 & 409 \\
$\last$  &  92   & 629 & 1501 & 878 & 2111 & 3517  & 728 & 697 & 932
\end{tabular}
\\
\begin{tabular}{c|cccccccccccccccc}
 $M$   & 13      &       14 & 15      & 16     &    17  & 18       & 19     &20 \\ \hline
$\triad$   & 409    & 409     & 409    & 409   & 409   & 409     & 409   & 409\\
$\last$      & 11212 & 20249 & 7335 & 4343 & 3632 & 17862 & 5533 & 43875
\end{tabular}
\end{table}

The CT for IA has been computed in 1983 in Allen's famous work. When applying Algorithm~\ref{algorithm1} to IA, we do not consider all intervals. Instead, we restrict the domain to the set of all intervals contained in $[0,M)$ that have integer nodes
$$D_{M}=\{[p,q]|p,q\in \mathbb{Z}, 0\leq p<q < M\},$$
and use uniform distribution to choose random intervals. It is easy to see that the size of the domain is $M(M-1)/2$.
Note that to converge fast and generate all entries, we need to choose an appropriate $M$. If $M$ is too small, then it is possible that some c-triads can not be instantiated. On the other hand, if $M$ is too big, relations that require one or more exact matches (such as $\sfm$ in IA and $\sfm^=$ in the INDU calculus to be introduced in the next example) is very hard to generate, i.e. the probability of generating such an instance is very small. For a new qualitative calculus, there is no general rules for choosing $M$. Usually, pilot experiments are necessary to better understand the characteristics of the calculus. 

Table~\ref{tab:impl_int} summarises the results for $M=4$ to $M=20$. In the experiment, we generate one million instances of triples of elements for each domain $D_M$.  In all cases the dynamic table becomes stable in less than 50,000 loops. When the table becomes stable, the numbers of triads computed are not always the correct one (that is 409). This is mainly because the domain is too small. For $M$ bigger than or equal to six, we always get the correct number of triads.\footnote{The 3-completeness of $D_6$ follows from the fact that each consistent IA network involving three variables has a solution in $D_6$.} The loops needed (i.e. \last) vary from less than a thousand to more than 43 thousand (see Table~\ref{tab:impl_int}). In general, the smaller the domain is the more efficient the algorithm is.

\begin{table}[htb]\centering
\caption{Implementation for INDU, where $\triad$ is the number of c-triads recorded by running the algorithm on $D_M$ for $M=6$ to $M=20$, $\last$ is the loop when the last triad is first recorded}\label{tab:impl_indu}
\scalebox{1}{
\begin{tabular}{c|ccccccccccc}
$M$       &   6     & 7       &    8      &    9      & 10       & 11      &   12       & 13     \\ \hline
$\triad$ & 1045  & 1531 & 1819   & 1987   & 2041   & 2053  & 2053     & 2053\\
$\last$  &  3766 & 5753 & 10417 & 35201 &35891  & 25031& 12512   & 27728
\end{tabular}
}
\\
\scalebox{1}{
\begin{tabular}{c|ccccccccccc}
$M$ &        14       & 15     & 16      & 17     &  18   & 19      & 20  \\ \hline
$\triad$   & 2053   & 2053  & 2053 & 2053 & 2053 & 2053  & 2053\\
$\last$ & 17223 &24578&14758   &22491&29034 &49693 & 19772
\end{tabular}
}

\end{table}

\begin{example}[INDU calculus]\label{ex:indu}
The INDU calculus \cite{PujariKS99} is a refinement of IA. For each pair of intervals $a,b$, INDU allows us to compare the durations of $a,b$. This means, some IA relations may be split into three sub-relations. For example, $\sfp$ is split into three relations $\sfp^<,\sfp^=,\sfp^>$. Similar situations apply to $\sfm, \sfo, \sfoi,\sfmi$, and $\sfpi$. The other seven relations have no proper sub-relations.  Therefore, INDU has 25 basic relations.
\end{example}

INDU is quite unlike IA. For example, it is not closed under composition, and a path-consistent basic network is not necessarily consistent \cite{BalbianiCL06}.

Applying our algorithm to INDU, we use the same subdomain $D_M$ as for IA. From Table~\ref{tab:impl_indu} we can see that $D_6$ is no longer 3-complete: more than 1000 c-triads do not appear in the stable table. The table becomes complete in $D_{11}$, which has 2053 c-triads. The 3-completeness of $D_{11}$ is confirmed by the following proposition. 
\begin{proposition}\label{prop:indu-2053}
The INDU calculus has at most 2053 c-triads.
\end{proposition}
\begin{proof}[Sketch]
For any three INDU relations $\alpha^{\star_1},\beta^{\star_2}, \gamma^{\star_3}$  $(\star_1,\star_2,\star_3\in\{<,=,>\}$), it is easy to see that $\langle \alpha^{\star_1},\gamma^{\star_2},\beta^{\star_3}\rangle$ is a valid c-triad of INDU only if $\langle \alpha,\gamma,\beta\rangle$ is a valid c-triad of IA and $\langle\star_1,\star_2,\star_3\rangle$ is a valid c-triad of PA. We note that for IA relations in $\{\sfd,\sfs,\sff,\sfeq,\sfsi,\sffi,\sfdi\}$, only $\sfd^<,\sfs^<,\sff^<,\sfeq^=,\sfsi^>,\sffi^>,\sfdi^>$ are valid INDU relations. It is routine to check that there are only 2053 triples of INDU relations that satisfy the above two constraints. We recall that IA has 409 c-triads (see Table~\ref{tab:impl_int}), and PA has 13 c-triads (see Table~\ref{tab:PA}).  \qed
\end{proof}
Since 2053 valid c-triads are recorded by running the algorithm on $D_{11}$ for INDU, we know INDU has precisely 2053 c-triads, and $D_{11}$ is 3-complete for INDU. Moreover, we have that $\langle \alpha^{\star_1},\gamma^{\star_2},\beta^{\star_3}\rangle$ is a valid c-triad of INDU \emph{if and only if} $\langle \alpha,\gamma,\beta\rangle$ is a valid c-triad of IA and $\langle\star_1,\star_2,\star_3\rangle$ is a valid c-triad of PA.  

It seems that this is the first time that the CT of INDU has been computed. 

\subsection{The Oriented Point Relation Algebra}
In the $\opra_m$ calculus, where $m$ is a parameter characterizing its granularity, each object is represented as an oriented point (\emph{o-point} for short) in the plane. Each o-point has an orientation. Based on which, $2m-1$ other directions are introduced according to the chosen granularity. Any other o-point is located on either a ray or in a section between two consecutive rays. Each of these rays and sections is assigned an integer from 0 to $4m-1$. The relative directional information of two o-points $A,B$ is uniquely encoded in a pair of integer numbers $(s,t)$, where $s$ is the ray or section of $A$ in which $B$ is located, and $t$ is the ray or section of $B$ in which $A$ is located. Such a relation is also written as $A {\mst m s t} B$.
In the case that the locations of $A$ and $B$ coincide, the relation between $A$ and $B$ is written as ${\mst m s {}}B$, where $s$ is the ray or section of $A$ in which the orientation of $B$ is located. Therefore, there are $4m(4m+1)$ basic relations in $\opra_m$.

\begin{figure}[htb]
\centering
\begin{tabular}{cc}
\includegraphics[width=.45\textwidth]{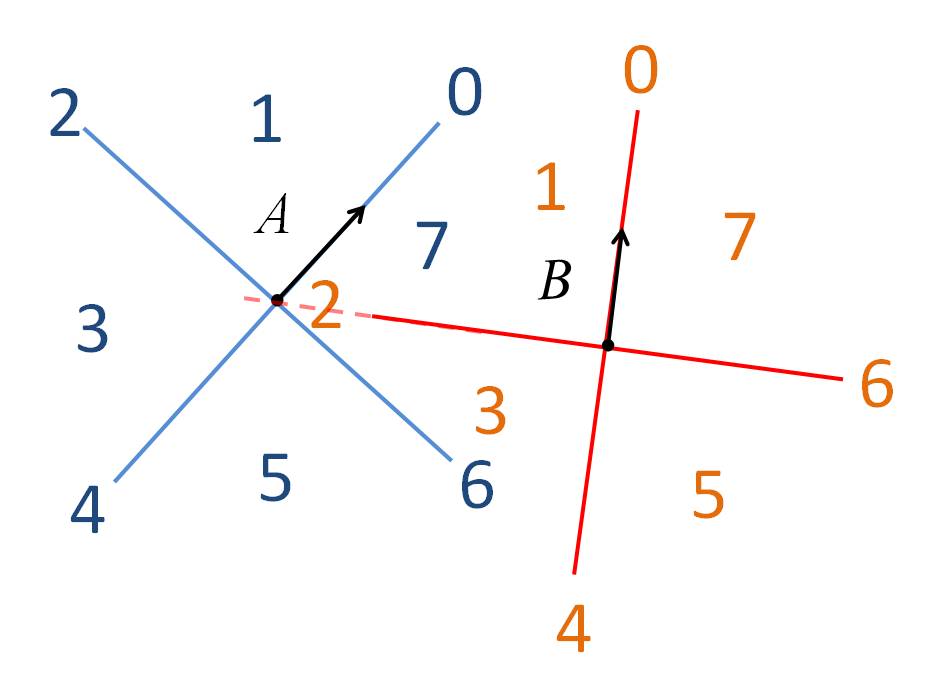}
&
\includegraphics[width=.3\textwidth]{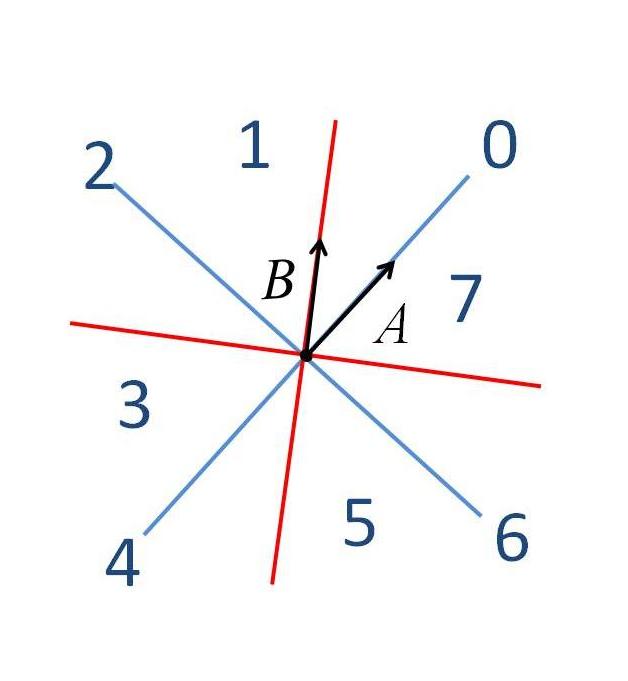}
\\
(a) & (b)
\end{tabular}
\caption{Two o-points $A,B$ with the $\opra_2$ relation (a) $\mst 2 7 2$ and (b) $\mst 2 1 {}$.}
\label{fig:opra-ex1}
\end{figure}

There are two natural ways to represent o-points. One uses the Cartesian coordinate system, the other use polar coordinate system. We next show the choice of coordinate system will significantly affect the experimental results, which are compared with that of \cite{MM10}.

In the Cartesian coordinate system, an o-point $P$ is represented by its coordination $(x,y)$ and its orientation $\phi$.
\begin{definition}\label{dfn:D(M1,M2)}
Let $M_1$ and $M_2$ be two positive integers. We define a Cartesian based subdomain of $\opra_m$ as
\begin{equation*}
\label{eq:D(M1,M2)}
D_c(M_1,M_2)=\{((x,y),\phi):x,y\in [-M_1,M_1]\cap\mathbb{Z},\phi\in\Phi_{M_2}\},
\end{equation*}
where $\Phi_{M_2}\equiv\{0, 2\pi/M_2, \cdots, (M_2-1)/M_2\times 2\pi\}$.
\end{definition}

\begin{table}[htb]\centering
\caption{Implementation for $\opra_1$ on a Cartesian coordinated domain $D_c(M_1,M_2)$, where $\triad$ is the number of c-triads computed by running the algorithm on $D_c(M_1,M_2)$ for $M_1=6$; \last\ is the loop when the last triad is first recorded for $M_2=8$ (in the 2nd last row) and $M_2=16$ (in the last row)}\label{tab:impl_opra1}
\begin{tabular}{c|ccccccccc}
$M_2$   &2  &3   &4   &5   &6   &8   &10    &12  &16\\ \hline
$\triad$&148&1024&1056&1024&1024&\textbf{1440}&1024&1408&\textbf{1440}
\end{tabular}
\\
\begin{tabular}{c|ccccc}
$M_1$          &2     &4      &6      &8        &10       \\ \hline
\last\ ($M_2=8$) &8082  &35932  &411893 &881787   &$>1000000$ \\
\last\ ($M_2=16$)&18618 &295936 &174490 &$>1000000$ &$>1000000$
\end{tabular}
\end{table}

Our experimental results show that, for $\opra_1$, the algorithm converges and generates the correct CT for subdomains with $M_1\geq 2$ and $M_2\in\{8,16\}$. That is, the smallest 3-complete subdomain is $D_c(2,8)$.

For $\opra_2$, however, the algorithm does not compute the desired CT in ten million loops. Actually, it is impossible to compute the desired CT if we use Cartesian coordination. Consider the following example. Suppose $A,B,C$ are three o-points, such that $\triangle ABC$ is an acute triangle, and the orientation of $A$ is the same as the direction from $A$ to $B$, the orientations of $B$ and $C$ are similar. In this configuration, we have  $A \mst 2 0 1 B$,  $B \mst 2 0 1 C$, and $A \mst 2 1 0 C$. This configuration, however, cannot be realised in a Cartesian based subdomain.\footnote{The proof of this statement is much involved and omitted in this paper.}

\begin{table}\centering
\caption{Implementation for $\opra_2$ on a Cartesian coordinated domain $D_c(M_1,M_2)$, where $\triad$ is the number of c-triads computed by running Algorithm~\ref{algorithm1} ten million times on $D_c(M_1,M_2)$ for $M_1=6$}\label{tab:impl_opra2}
\scalebox{1}{
\begin{tabular}{c|ccccccc}
$M_2$   &2    &4  &6    &8    &10   &12    &16\\ \hline
$\triad$&2704&2704&21792&23616&21792&21792 &35232
\end{tabular}
}
\end{table}

Based on the above observation, we turn to the polar coordinated representation. In the polar coordinate system, an o-point $P$ is represented by its polar coordination $(\rho,\theta)$ and its orientation $\phi$.
\begin{definition}\label{dfn:Do(M1,M2)}
Let $M_1$ and $M_2$ be two positive integers. We define a polar coordinated subdomain of $\opra_m$ as
\begin{equation*}
\label{eq:D(M1,M2)}
D_p(M_1,M_2)=\{((\rho,\theta),\phi):\rho\in [0,M_1]\cap\mathbb{Z},\theta,\phi\in\Phi_{M_2}\},
\end{equation*}
where $\Phi_{M_2}\equiv\{0, 2\pi/M_2, \cdots, (M_2-1)/M_2\times 2\pi\}$.
\end{definition}

As in Cartesian based subdomains, the parameter $M_2$ determines if a domain is complete, while $M_1$ determines the efficiency of the algorithm. For  $\opra_1$, we have $D(M_1,M_2)$ is a 3-complete subdomain if $M_1\geq 2$ and $M_2=6,8,10,12,16$ (see Table~\ref{tab:impl_opra1b-polar}); for $\opra_2$, we have $D(M_1,M_2)$ is 3-complete if $M_1\geq 4$ and $M_2=6,10,12,16$ (see Table~\ref{tab:impl_opra2-polar}).

\begin{table}[htb]\centering
\caption{Implementation for $\opra_1$ on a polar coordinated domain $D_p(M_1,M_2)$, where $\triad$ is the number of c-triads computed by running the algorithm on $D_p(M_1,M_2)$ for $M_1=6$; \last\ is the loop when the last triad is first recorded for $M_2=8$ (in the 2nd last row) and $M_2=16$ (in the last row)}
\label{tab:impl_opra1b-polar}
\scalebox{1}{
\begin{tabular}{c|ccccccccc}
$M_2$   &2  &3    &4    &5    &6           &8           &10          &12          &16 \\ \hline
$\triad$&52 &1024 &1032 &1408 &\textbf{1440} &\textbf{1440} &\textbf{1440} &\textbf{1440} &\textbf{1440}
\end{tabular}
}

\scalebox{1}{
\begin{tabular}{c|ccccc}
$M_1$            &4     &6     &8      &10    &16\\ \hline
\last\ ($M_2=8$) &3072	&4868  &22327  &10363 &38843\\
\last\ ($M_2=16$)&26219	&45831 &121542 &71205 &146536
\end{tabular}
}

\end{table}

\begin{table}[htb]\centering
\caption{Implementation for $\opra_2$ on a polar coordinated domain $D_p(M_1,M_2)$, where $\triad$ is the number of c-triads computed by running the algorithm on $D_p(M_1,M_2)$ for $M_1=6$}\label{tab:impl_opra2-polar}
\scalebox{1}{
\begin{tabular}{c|cccccccc}
$M_2$   &2   &3     &4    &6            &8     &10           &12           &16 \\ \hline
$\triad$&400 &24672 &2128 &\textbf{36256} &23616 &\textbf{36256} &\textbf{36256} &\textbf{36256}
\end{tabular}
}

\end{table}

\subsection{The Region Connection Calculus}
Our algorithm works very well for simple objects like points and intervals. We next consider a region-based topological calculus RCC-8. It is worth noting that an automated derivation of the composition table was reported in \cite{Egenhofer94} for a similar calculus (the 9-intersection model).

\begin{example}[RCC-8 algebra]\label{ex:RCC-8}
Let $U$ be the set of bounded plane regions (i.e. nonempty regular closed sets in the plane). Five binary relations are defined in Table~\ref{tab:RCC-8}. The RCC-8 algebra \cite{RandellCC92} is the Boolean algebra generated by these five relations, the identity relation \beq, and the converses of \btpp\ and \bntpp.
\end{example}

\begin{table}[htb]\centering
\caption{A topological interpretation of basic RCC-8 relations in the plane, where
$a,b$ are two bounded plane regions, and $a^\circ,b^\circ$ are the interiors of $a,b$, respectively.}\label{tab:RCC-8}
\scalebox{1}{
\begin{tabular}{c|c}
  Relation & Meaning  \\ \hline
 \bdc & $a\cap b=\varnothing$  \\
 \bec & $a\cap b\not=\varnothing$, $a^\circ\cap b^\circ=\varnothing$ \\
 \bpo & $a\not\subseteq b$, $b\not\subseteq a$, $a^\circ\cap b^\circ\not=\varnothing$\\
 \btpp & $a\subset b$, $a\not\subset b^\circ$ \\
 \bntpp & $a\subset b^\circ$\\
\end{tabular}
}
\end{table}

\begin{table}\centering
\caption{Implementation for RCC-8, where $\triad$ is the number of c-triads computed by running the algorithm on $D_M$ using rectangles, \last\ is the loop when the last triad is first recorded}\label{tab:impl_rcc-8-rectangle}
\scalebox{1}{
\begin{tabular}{c|cccccccccccccccc}
$M$        &       4   &      5  &      6       &    8      &    10      &   15        &    20 \\ \hline
$\triad$&     114 & 177   &  192         & 192     &  192       & 192        & 192\\
$\last$ & 14776 & 6513 & 2332646 &  56067 &  198255 &  261729 & 1521173
\end{tabular}}
\end{table}

Plane regions are much more complicated to represent than intervals or o-points. In most cases they are approximated by polygons or digital regions (i.e., a subset of $\mathbb{Z}^2$). Furthermore, it is natural to take a shot on simple objects at the beginning, since they are easy to deal with and important in applications. For RCC-8, we make experiments over two subdomains: rectangles and disks. The experiments show that these subdomains are good enough for our purpose, but when necessary, we could also consider general polygons or bounded digital regions.

We first consider subdomains whose elements are rectangles sides of which are parallel to the two axes. We introduce one parameter $M$, and require the four nodes be points in $[0,M)\times[0,M) \cap \mathbb{Z}^2$. The complete RCC-8 CT has 193 table entries. Since $\beq\circ\beq=\beq$, we know $\langle\beq,\beq,\beq\rangle$ is a c-triad. The other 192 c-triads can be confirmed using our algorithm. In Table~\ref{tab:impl_rcc-8-rectangle}, we show the results of running the algorithm 10 million times and require $M$ vary from 4 to 20. We can see from the table that $D_M$ is a 3-complete subdomain only if $M\geq 6$.

\begin{table}[htb]\centering
\caption{Implementation for RCC-8, where $\triad$ is the number of c-triads computed by running the algorithm on $D_M$ using disks, \last\ is the loop when the last triad is first recorded}\label{tab:impl_rcc-8-disk}
\scalebox{1}{
\begin{tabular}{c|cccccccccccccccc}
$M$        &       4   &      5  &      6       &    8       &      10      &   15        &    20 \\ \hline
$\triad$ &    188  & 192  &   192       &192      &  192     & 192        & 192\\
$\last$  & 1759   & 8913 &   9489    & 25955 &   113757&  942914 & 2961628
\end{tabular}}

\end{table}

We next consider subdomains consisting of disks (see Table~\ref{tab:impl_rcc-8-disk}). We introduce one parameter $M$, and require $x,y\in [0,M]\cap\mathbb{Z}$, $r\in [1,M]\cap\mathbb{Z}$, where $(x,y)$ and $r$ are, respectively, the centre and the radius of the closed disk $B((x,y),r))$. In this case, $M=5$ is good enough to generate all c-triads. We notice that the number of loops needed (i.e. \last) increases quickly as $M$ increases. For example, when $M=20$, the dynamic table becomes stable after nearly 3 million loops. This is mainly due to that an instance of the c-triad $\langle\bntpp,\bntpp,\bntpp\rangle$ is very hard to generate. The `hard' c-triad is, however, easy to prove.






\section{Further Discussions}
In the last section, we have applied our algorithm to generate the CTs of IA, INDU, $\opra_1$, $\opra_2$, and RCC-8. In this section, we discuss the advantages and disadvantages of our method. The algorithm works very well for simple objects like points, intervals, rectangles, and disks, especially in a small subdomain. For a qualitative calculus with less than 100 basic relations, it can compute the CT in a few minutes.

We also considered larger calculi. The Oriented Point Relation Algebras $\opra_3$ and $\opra_4$ have, respectively, 156 and 272 basic relations. Applying our algorithm to an appropriate polar coordinated subdomain $D_p(M_1,M_2)$, 261,576 and 1,082,752 c-triads, respectively, have been found in a few hours, which coincide with those computed in \cite{MM10}. This implies that the corresponding subdomains are 3-complete.

For calculi defined over regions, the main obstacle of using our approach is the cost of generating random regions. For RCC-8, we circumvent this obstacle by taking rectangles and disks. But this circumvention does not work for the Cardinal Direction Calculus (CDC) \cite{Goyal00}, as the shape of the region matters in this calculus. The CDC contains 218 basic relations. We run our algorithm for the CDC on the subdomain containing digital regions in $[0,5]\times [0,5]$\footnote{The 3-completeness of this subdomain is confirmed by results reported in \cite{LiuZLY10}.}, using normal distribution. The result is not ideal. After one day, we have computed about 60\% of the total 2.2 million c-triads of the CDC. Improvements will be made later, adopting more appropriate or heuristic distribution.

In many applications of qualitative calculi, the objects used are often restricted. Take $\opra_2$ as example. In many real world applications, e.g. the  Interstate Highway System of the USA, oriented objects are all taken from a underlying graph. In these cases, each o-point has only a few possible directions. To support reasoning with this domain, we had better have a \emph{customised} $\opra_2$ calculus, together with a customised CT.  Our algorithm works perfect to this end. For example, consider the restriction of $\opra_2$ calculus on $\mathbb{Z}^2$. Each o-point in this calculus has integer coordinations and has one of the four orientations from $\{0,\pi/2,\pi,3\pi/2\}$. Using our algorithm, the CT for this customised calculus has been generated in a few minutes. Experiment result shows that this customised calculus has 2704 c-triads.


Even for the well-known Interval Algebra, our algorithm suggests a new viewpoint for efficient reasoning. We note that, based on our method, we can easily compute the \emph{probability} of each basic relation in the weak composition of any two basic IA relations. This may be used in approximate temporal reasoning, especially when
the application domain has a different structure than the universe of IA. Work towards this direction will be reported in another paper.

\section{Conclusion}
In this paper, we introduced a general and simple semi-automatic method for computing the composition tables of qualitative calculi. The described method is a very natural approach, and similar idea was used to derive composition tables for an elaboration of RCC with convexity \cite{Cohn+93}, and for a ternary directional calculus \cite{Clementini+10}. The table computed in \cite{Cohn+93} was acknowledged there as incomplete. The table computed in \cite{Clementini+10} is complete, but its completeness was guaranteed by manually checking all geometric configurations that satisfy the table.  Except these two works, very little attention has been given to this natural approach in the literature on composition tables. We think a systematic examination is necessary to discover both the strong and weak points of this approach.

We implemented the basic algorithm for several well-known qualitative calculi, including the Interval Algebra, INDU, $\opra_m$ for $m=1\sim 4$, and RCC-8. Our experiments suggest that the proposed method works very well for point-based calculi, but not so well for region-based calculi. In particular, we established, as far as we know, for the first time the correct CT for INDU, and confirmed the validity of the algorithm reported for the $\opra$ calculi \cite{MM10}. Our method can be easily integrated into existing qualitative solvers e.g. SparQ \cite{WallgrunFWDF06} or GQR \cite{WestphalWG09}. This provides a partial answer to the challenge proposed in \cite{Cohn95}.

Recently, Wolter proposes (in an upcoming article \cite{Wolter11}) to derive composition tables by solving systems of polynomial (in)equations over the reals. This approach works well for several point-based calculi, but not always generates the complete composition table.  

Our method relies on the assumption that the qualitative calculus has a small `discretised' 3-complete subdomain. All calculi considered in this paper satisfy this property. It is still open whether all interesting calculi appeared in the literature satisfy this property. Future work will also discuss the applications of our method for reasoning with a customised composition table.

\bibliographystyle{splncs}
\bibliography{proco}
\end{document}